%% file: emnlp2021.tex
\pdfoutput=1

\documentclass[11pt]{article}

\usepackage[]{emnlp2021}

\usepackage{times}
\usepackage{latexsym}

\usepackage[T1]{fontenc}

\usepackage[utf8]{inputenc}

\usepackage{microtype}

\usepackage{xspace,mfirstuc,tabulary}

\usepackage{amsmath}
\usepackage{amssymb}
\usepackage{amsthm}
\usepackage{amsfonts}
\usepackage{mathtools}
\usepackage{enumitem}
\usepackage{relsize}

\usepackage{subcaption}
\usepackage{tikz}
\usetikzlibrary{shapes,arrows}
\usetikzlibrary{arrows.meta}
\usetikzlibrary{shapes.arrows}
\usetikzlibrary{shapes.misc}
\usetikzlibrary{fit}

\usepackage{emoji}

\newcommand{\ucambridge}{\emoji[emojis]{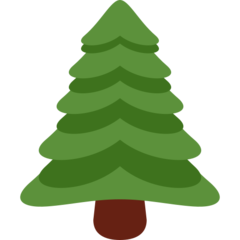}}
\newcommand{\ethz}{\emoji[emojis]{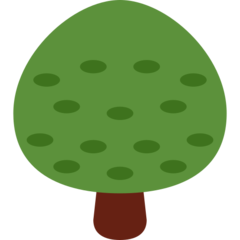}}
\newcommand{\jhu}{\emoji[emojis]{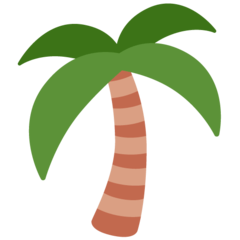}}

\usepackage{inconsolata}
\author{
{Ran Zmigrod\raise1.0ex\hbox{\normalfont\ucambridge}\raise1.0ex\hbox{\normalfont}}~\;~Tim Vieira\raise1.0ex\hbox{\normalfont\jhu}~\;~Ryan Cotterell\raise1.0ex\hbox{\normalfont\ucambridge,\ethz}
\\
  \raise1.0ex\hbox{\normalfont\ucambridge}University of Cambridge~\;~\raise1.0ex\hbox{\normalfont\jhu}Johns Hopkins University~\;~\raise1.0ex\hbox{\normalfont\ethz}ETH Z\"{u}rich \\
  \href{mailto:rz279@cam.ac.uk}{\tt rz279@cam.ac.uk}~\;~\href{mailto:tim.f.vieira@gmail.com}{\tt tim.f.vieira@gmail.com} \\ \href{mailto:ryan.cotterell@inf.ethz.ch}{\tt ryan.cotterell@inf.ethz.ch}
}

\date{}

\newtheorem{thm}{Theorem}

\newtheorem{lemma}{Lemma}

\theoremstyle{definition}

\usepackage{algorithm}
\usepackage[noend]{algpseudocode}  
\algrenewcommand\algorithmicindent{1.0em}%

\newcommand{\rightcomment}[1]{{\color{gray} \(\triangleright\) {\footnotesize\textit{#1}}}}
\algrenewcommand{\algorithmiccomment}[1]{\hfill \rightcomment{#1}}  
\algnewcommand{\LineComment}[1]{\State\rightcomment{#1}}
\algnewcommand{\LinesComment}[1]{\State\rightcomment{\parbox[t]{.95\linewidth-\leftmargin-\widthof{\(\triangleright\) }}{#1}}}

\algrenewcommand\alglinenumber[1]{{\tiny\color{black!50}#1.}\hspace{-2pt}}

\newcommand{\algorithmicfunc}[1]{\textbf{def} {#1}:}
\algdef{SE}[FUNC]{Func}{EndFunc}[1]{\algorithmicfunc{#1}}{}
\makeatletter
\ifthenelse{\equal{\ALG@noend}{t}}%
  {\algtext*{EndFunc}}
  {}%
\makeatother

\usepackage{inconsolata}
\usepackage{cleveref}
\crefname{section}{\S}{\S\S}
\Crefname{section}{\S}{\S\S}
\crefname{table}{Tab.}{}
\crefname{figure}{Fig.}{}
\crefname{algorithm}{Alg}{}
\crefname{algorithm}{Alg}{}
\crefname{line}{line}{lines}
\Crefname{line}{Line}{Lines}
\crefname{appendix}{App.}{}
\crefformat{section}{\S#2#1#3}
\crefname{thm}{Theorem}{}
\crefname{prop}{Proposition}{}
\crefname{defin}{Definition}{}
\crefname{lemma}{Lemma}{}
\crefname{cor}{Corollary}{}
\crefname{equation}{}{}

\input{macros}

\usepackage{tagpdf}
\tagpdfsetup{uncompress, 
             activate-mc}

\allowdisplaybreaks

\title{Efficient Sampling of Dependency Structures}

\begin{document}
\maketitle

\begin{abstract}
Probabilistic distributions over spanning trees in directed graphs are a fundamental model of dependency structure in natural language processing, syntactic dependency trees.
In NLP, dependency trees often have an additional root constraint: only one edge may emanate from the root.
However, no sampling algorithm has been presented in the literature to account for this additional constraint.
In this paper, we adapt two spanning tree sampling algorithms to  sample dependency trees from a graph subject to the root constraint.
\citet{wilson96}'s sampling algorithm has a running time of $\bigo{\hittingTime}$ where $\hittingTime$ is the mean hitting time of the graph.
\citet{colbourn96}'s sampling algorithm has a running time of $\bigo{\nN^3}$, which is often greater than the mean hitting time of a directed graph.
Additionally, we build upon Colbourn's algorithm and present a novel extension that can sample $\nK$ trees without replacement in $\bigo{\nK\nN^3+\nK^2\nN}$ time.
To the best of our knowledge, no algorithm has been given for sampling spanning trees without replacement from a directed graph.\footnote{Our implementation of these algorithms is publicly available at \url{https://github.com/rycolab/treesample}.}

\end{abstract}

\section{Introduction}
Spanning trees in directed graphs\footnote{\emph{Directed} spanning trees are known as \emph{\ars} in the graph theory literature \citep{williamson}.  However, we will simply refer to them as spanning trees.} are fundamental combinatorial structures in natural language processing 
where they are used to represent dependency structures---especially syntactic dependency structure \citep{dp}.
Additionally, probabilistic models over spanning trees 
are common in the NLP literature with applications primarily in
non-projective dependency parsing \citep{pei-etal-2015-effective, wang-chang-2016-graph, dozat, ma},
but also in recovering phylogenic structures \citep{andrews-etal-2012-name}, 
and event extraction \citep{mcclosky-etal-2011-event}.

Given the prevalence of such probabilistic models, efficient dependency tree sampling algorithms deserve study.
Indeed, some work has been done in transition-based dependency parsing \citep{keith-etal-2018-monte} as well as graph-based dependency parsing \citep{nakagawa-2007-multilingual, marecek-zabokrtsky-2011-gibbs}.
Sampling has also been utilized in an abundance of NLP tasks,
such as text generation \citep{clark-etal-2018-neural, fedusGD18}, 
co-reference resolution \citep{singh-etal-2012-monte},
and language modeling \citep{mnihH07, logan-iv-etal-2020-importance}.

The theoretical computer science literature has several efficient algorithms for sampling directed spanning trees.
These algorithms come in two flavors.
First, random walks through Markov chains have been used to sample spanning trees from both undirected \citep{broder89, aldous90} and directed graphs \citep{wilson96}.
The algorithm of \citet{wilson96} is linear in the mean hitting time of the graph and is currently the fastest sampling algorithm for directed spanning trees.
It has been used in dependency parsing inference by \citet{zhang-etal-2014-greed, zhang-etal-2014-steps}.
Second, several algorithms have leveraged the matrix--tree theorem \citep[MTT;][]{kirchhoff, tutte1984graph}.
The MTT has been frequently used to perform inference on non-projective graph-based dependency parsers \citep{koo-et-al-2007, mcdonald-satta-2007,smith-smith-2007, zmigrod-2021-tacl}.
This theorem was first used for sampling by \citet{guenoche83} who gave an $\bigo{\nN^5}$ algorithm which was then improved by \citet{kulkarni90} and \citet{colbourn96}.
\citet{colbourn96} give an $\bigo{\nN^3}$ algorithm to sample spanning trees from an \emph{unweighted} directed graph.  We generalize their algorithm to the weighted case.

While directed spanning tree sampling algorithms exist, an important constraint of many dependency tree schemes, such as the Universal Dependency (UD) scheme \cite{ud}, is that a dependency tree may only have one edge emanating from the designated root symbol.
Algorithms exists for enforcing this constraint in decoding \citep{GabowT84, zmigrod-etal-2020-please, stanojevic-cohen-2021} and inference \citep{koo-et-al-2007, zmigrod-2021-tacl}.
However, to the best of our knowledge, no sampling algorithm exists which enforces the root constraint.

In this paper, we adapt the algorithms of \citet{wilson96} and \citet{colbourn96}\footnote{Note that \citet{colbourn96} presents two algorithms for sampling directed spanning trees.  In this work, we focus on their first algorithm, which runs in $\bigo{\nN^3}$. While the second algorithm is based on a reduction to fast matrix multiplication, which is typically impractical.
We did not extend this algorithm because it is not amendable to sampling without replacement and it is generally slower than the algorithm of \citet{wilson96}.} to efficiently sample directed spanning trees subject to a root constraint while maintaining the runtime of the original algorithms.
While our modification to \citet{colbourn96}'s algorithm faithfully samples trees subject to the root constraint, our extension to \citet{wilson96}'s algorithm produces biased samples.
Additionally, we provide a further extension to \citet{colbourn96}'s algorithm that allows us to sample trees without replacement.
Sampling without replacement (SWOR) algorithms are useful when distributions are skewed, which is often the case in a trained system.
To the best of our knowledge, no SWOR algorithm has been presented in the literature for directed spanning trees, though \citet{shi2020} provides a general framework that enables a SWOR algorithm to be adapted for particular kinds of sampling algorithms.

\section{Distributions over Trees} 
We consider distributions over spanning trees in \defn{rooted directed weighted graphs} (\defn{graphs} for short).
A graph is denoted by $\graph=(\root, \nodes, \edges)$ where $\nodes$ is a set of $\nN+1$ nodes including a designated root node $\root$ and $\edges$ is a set of ordered pairs between two nodes $\edgeij$.
We note that the non-root nodes $\nodes\setminus\{\root\}$ can be enumerated as $[1,\dots,\nN]$.
For each edge $\edgeij\in\edges$, we associate a non-negative weight $\wedgeij\in\nonnegreal$.
Note that $\wedgeij=0$ when $\edgeij\not\in\edges$.

A \defn{directed spanning tree}, denoted by $\tree$, is a collection of edges in a graph $\graph$ such that each node $j\in\nodes\setminus\{\root\}$ has exactly one incoming edge and $\tree$ contains no cycles.
Moreover, we specifically examine distributions over \defn{dependency trees}, which are spanning trees with an additional constraint that exactly one edge must emanate from the root $\root$.
This additional constraint is common amongst most dependency tree annotation schemes, e.g., Universal Dependencies.\footnote{However, we note that there are exceptions that do not require the root constraint, such as the Prague Treebank \citep{prague_dep}.}
When the type of tree (spanning or dependency) is clear from context, we will simply use \defn{trees}.
The set of all spanning trees of a graph $\graph$ is given by $\treesg{\graph}$ and the set of all dependency trees of a graph $\graph$ is given by $\dtreesg{\graph}$.
Note that $\dtreesg{\graph}\subseteq\treesg{\graph}$.
When the graph $\graph$ is clear from context, we will refer to these sets $\trees$ and $\dtrees$ respectively.

The \defn{weight} of a tree $\tree$ is the product of its edge weights:
\begin{equation}
    \weight{\tree} \defeq \monsterprod{\edgeij\in\tree}\,\wedgeij
\end{equation}
The \defn{probability} of a spanning tree is then given by
\begin{equation}
    \prob(\tree) \defeq \frac{\weight{\tree}}{\Z} 
\quad\text{where}\quad
    \Z \defeq \sum_{\tree\in\trees} \weight{\tree}
\end{equation}
The normalizing constant in the case of dependency trees sums over $\dtrees$ instead of $\trees$.\looseness=-1

\input{figures/walk}

\section{Random Walk Sampling}\label{sec:wilson}
In this section, we present the spanning tree sampling algorithm of \citet{wilson96} and adapt it to sample dependency trees.
The algorithm is based on a random walk through the nodes of the graph until a tree is formed.
In order to do this, we require a graph $\graph$ to be \defn{stochastic}.
A stochastic graph is such that the weights of all incoming edges to a node sum to one.
That is, for all non-root nodes $j\in\nodes\setminus\{\root\}$ we have that
\begin{equation}
    \sum_{i\in\nodes}\wedge{i}{j} = 1
\end{equation}
\noindent \citet{wilson96} shows that any graph may be converted into a stochastic graph, by adjusting the edge weights to be locally normalized along all incoming edges.
Therefore, the edge $\edgeij$ has a weight $w'_{\bedgeij}$ defined by
\begin{equation}
    w'_{\bedgeij} \defeq \frac{\wedgeij}{\sum_{i'\in\nodes}\wedge{i'}{j}}
\end{equation}
We can also define this new weight over a tree $w'(\tree)\defeq\prod_{\edgeij\in\tree}w'_{\bedgeij}$.

\begin{lemma}\label{lem:prop}
For any tree $\tree\in\trees$,
\begin{equation}
    w'(\tree)\propto \weight{\tree}
\end{equation}
\end{lemma}
\begin{proof}
\begin{subequations}
\begin{align}
    w'(\tree) &= \prod_{\edgeij\in\tree} w'_{\bedgeij} \\
    &= \prod_{\edgeij\in\tree} \frac{\wedgeij}{\sum_{i'\in\nodes}\wedge{i'}{j}} \\
    &=  \!\left(\!\prod_{\edgeij\in\tree} \!\frac{1}{\sum_{i'\in\nodes}\wedge{i'}{j}}\!\right) \!\!\!\! \left(\!\prod_{\edgeij\in\tree} \wedgeij\!\right) \nonumber \\
    &= \!\! \underbrace{\left(\!\prod_{\edgeij\in\tree} \!\frac{1}{\sum_{i'\in\nodes}\wedge{i'}{j}}\!\right)}_{\mathrm{constant}} \! \weight{\tree}  \label{eq:constant} \\
    &\propto \weight{\tree}
\end{align}
\end{subequations}
Note that the left-hand term in \cref{eq:constant} is a constant since every non-root node has in incoming edge in the tree, and so the constant is equal to
$\prod_{i\in\nodes\setminus\{\root\}}\left(\sum_{i'\in\nodes}\wedge{i'}{j}\right)^{-1}$.
\end{proof}

\input{figures/wilson}

A stochastic graph
then defines a \defn{Markov chain} that we can perform a random walk on.
We begin with a tree that is just composed of the root $\root$.
Then while there exists a node connect to the tree, we start a random walk where we add each edge used to the tree, until we encounter a node that is connected to the tree.
When we have connected to the tree, we can proceed to start a new random walk from a node not in the tree.

Of course, during a random walk, we may go through a cycle.
Whenever we walk through a cycle, we simply forget the cycle as part of the walk.
That is, if we find ourselves visiting node $i$ in our walk for a second time, we erase the cyclic path formed at $i$ and continue our walk.
This can be seen visually in \cref{fig:walk}.
This type of walk is known as a loop-erased random walk \citep{lawler1979self}.
Given we sample a sequence of edges $\mathcal{S}$ in a random walk, we can split the
We can split the edges into those corresponding to the loop-erased random walk, $\hat{\mathcal{S}}$, and those that do not, $\bar{\mathcal{S}}$.
\begin{lemma}\label{lem:ind}
For any sequence of edges $\mathcal{S}$ sampled from a random walk, we have
\begin{equation}
\prob(\mathcal{S}) = \prob(\hat{\mathcal{S}})\, \prob(\bar{\mathcal{S}})
\end{equation}
\end{lemma}
\begin{proof}
\begin{subequations}
\begin{align}
    \prob(\mathcal{S}) &= \smashoperator{\prod_{\edgeij\in\mathcal{S}}} w'_{\bedgeij} \\
    &= \left(\ \ \ \ \smashoperator{\prod_{\edgeij\in \hat{\mathcal{S}}}} w'_{\bedgeij}\right) \!\! \left(\ \ \ \ \smashoperator{\prod_{\edgeij\in\bar{\mathcal{S}}}} w'_{\bedgeij}\right) \nonumber \\
    &= \prob(\hat{\mathcal{S}})\, \prob(\bar{\mathcal{S}})
\end{align}
\end{subequations}
\end{proof}

Pseudocode for \citet{wilson96}'s algorithm  is given in \cref{alg:wilson}.
Examining the pseudocode, one can see it is possible to infinitely encounter cycles.
Fortunately, \citet{wilson96} proves that this is not the case and that the algorithm has a probabilistic bound of $\bigo{\hittingTime}$ where $\hittingTime$ is the mean hitting time of the graph.
The \defn{mean hitting time} of a Markov chain is defined as
\begin{equation}
    \hittingTime \defeq \sum_{i,j} \pi_i \pi_j h(i, j)
\end{equation}
where $\pi$ is the stationary distribution of the Markov chain, and $h(i, j)$ is the expected number of steps to reach node $j$ starting at node $i$.
\citet{wilson96} and \citet{broder89} demonstrate that the mean hitting time for directed graphs is usually small (sometimes as low as linear in $\nN$).
We will compare the empirical runtime of the algorithm against an $\bigo{\nN^3}$ algorithm in \cref{sec:experiment}.

\begin{thm}
\label{thm:wilson}
For any graph $\graph$, $\algCall{\wilsonAlg}{\graph}$ samples a directed spanning tree $\tree\in\trees$ with probability
\begin{equation}
    \prob(\tree)\propto\prod_{\edgeij\in\tree}\wedgeij
\end{equation}
Furthermore, $\wilsonAlg$ runs in $\bigo{\hittingTime}$ time.
\end{thm}
\begin{proof}
It is clear that if $\algCall{\wilsonAlg}{\graph}$ terminates, $\tree$ will contain a directed spanning tree.
Let $\mathcal{S}$ be the set of edges sampled in \cref{line:sample} of $\wilsonAlg$. 
As we sample $\tree$ using several independent random walks, by \cref{lem:ind}, we sample the edges in $\tree$ independently of the edges in $\mathcal{S}\setminus \tree$.
In particular, we sample $\tree$ with probability
\begin{equation}
    \prob(\tree) = \prod_{\edgeij\in\tree} w'_{\bedgeij}
\end{equation}
By \cref{lem:prop}, $\tree$ is a sample from our desired distribution over trees.

\citet[Theorem 3]{wilson96} proves that \cref{line:sample} of $\wilsonAlg$ executes $\bigo{H}$ times before the program terminates with tree $\tree$.
\end{proof}

\subsection{Root Constraint Modification}
\citet{wilson96}'s algorithm does not ensure that only one edge emanates from the root.
However, a simple modification to the algorithm allows us to do this.
The original algorithm samples a spanning tree rooted at $\root$ for a graph $\graph$.
Suppose we know we want the edge $\edge{\root}{j}$ to be the single edge emanating from the root.
Then we can run Wilson's algorithm on the graph $\graph'$, which is $\graph$ with node $\root$ removed and node $j$ defined as the new root.
By adding $\edge{\root}{j}$ to the newly sampled tree, we clearly have a dependency tree in $\graph$.
We can sample the root edge $\edge{\root}{j}$ from the root weights $\rootv$ to sample an unbiased dependency tree from the distribution.
The pseudocode for this algorithm is given as $\wilsonRCAlg$ in \cref{alg:wilson-rc}

\input{figures/wilson_rc}

\begin{thm}
For any graph $\graph$, $\algCall{\wilsonRCAlg}{\graph}$ samples a dependency tree from a biased distribution.\footnote{
Note that this theorem has been corrected following \citet{stanojevic2022}, who showed that $\wilsonRCAlg$ returns a biased sample by counter-example. We have included their counter-example into our proof.
They additionally provide an alternative solution to extend $\wilsonAlg$ to use rejection sampling in order to sample a dependency tree.
They show that on average $\wilsonAlg$ would be run three times to find a succesful dependency tree sample, thus maintaining the $\bigo{\hittingTime}$ runtime.
We refer the reader to \citet[Section 4]{stanojevic2022} for further details.
}
Furthermore, $\wilsonRCAlg$ runs in $\bigo{\hittingTime}$ in.
\end{thm}
\begin{proof}
We note that executing lines  \cref{line:start} to \cref{line:end} is equivalent to running $\wilsonAlg$ on the graph $\graph'$ that is rooted at $j$ (and does not have $\root$).
By \cref{thm:wilson}, this results in a tree $\tree'\in\treesg{\graph'}$.
As we then sample an edge emanating from the root, we sample a valid dependency tree $\tree$.
We show that $\tree$ is sampled from a biased distribution through an example.
Consider a graph with two edges emanating from the root, both with weight $\frac{1}{2}$.
Suppose that there exist two dependency trees using one of the edges, and only one dependency tree using the other edge.
Then we should sample the root edges with their \emph{marginal} probabilities $\frac{2}{3}$ and $\frac{1}{3}$ respectively.
However, we sample the root edges by their edge weight $\frac{1}{2}$.
Therefore, $\tree$ is sampled from a bias distribution.

Furthermore, as \cref{line:new} of $\wilsonRCAlg$ takes $\bigo{\nN}$ which is less than $\bigo{\hittingTime}$, by \cref{thm:wilson}, $\wilsonRCAlg$ has a runtime of $\bigo{\hittingTime}$.
\end{proof}

\input{figures/example}

\section{Ancestral Sampling}\label{sec:colbourn}
In this section, we present an extension to the ancestral sampling algorithm of \citet{colbourn96} to the weighted graph case.
This algorithm relies on the efficient computation of $\Z$ using the MTT \citep{kirchhoff, tutte1984graph},
allows us to compute $\Z$ in $\bigo{\nN^3}$ by taking the determinant of the Laplacian matrix, $\lap\in\real^{\nN\times\nN}$.
We use \citet{koo-et-al-2007}'s adaptation of the MTT to dependency trees.\footnote{\citet{koo-et-al-2007}'s adaptation constructs the Laplacian matrix without considering edges emanating from the root.
They then arbitrarily replace the first row of the Laplacian matrix with the root edge weights.
One can see that the first row is chosen for convenience by examining the proof of Proposition 1 in \citet{koo-et-al-2007}.
Indeed, the desired Laplacian can be obtained by replacing any row by the root edge weights.
}
\newcommand{\koocite}[0]{\citet[p. 140]{tutte1984graph}}
\begin{thm}[Proposition 1, \newcite{koo-et-al-2007}]
\label{thm:mtt}
For any graph $\graph$, the normalization constant $\Z$ over the distribution of dependency trees $\dtrees$ is given by $\Z = \abs{\lap}$ where
\begin{equation}\label{eq:lap}
    \lapelem{ij} =
    \begin{cases}
    \wedge{\root}{j} & \emph{\textbf{ if }} i=1 \\
    \ \ \ \ \ \ \monstersum{i'\in\nodes\setminus\{\root,i\}}\ \wedge{i'}{j} & \emph{\textbf{ if }} i=j \\
    -\wedgeij & \emph{\textbf{ otherwise}}
    \end{cases}
\end{equation}
\end{thm}
\noindent We present the algorithm using the above Laplacian matrix to sample dependency trees rather than spanning trees.
However, one can easily modify this algorithm to sample spanning trees.\footnote{\label{footnote:lap}The Laplacian matrix for spanning trees is given by $$\lapelem{ij}=\begin{cases}
   \sum_{i'\in\nodes\setminus\{i\}}\ \wedge{i'}{j} & \textbf{ if } i=j \\
    -\wedgeij & \textbf{ otherwise}
    \end{cases}$$}
The premise of the algorithm is that we iteratively sample an incoming edge to a non-root node of the graph until we have a tree.
Without loss of generality, we can enumerate the edges of any sampled tree $\tree$ as $e_1$ to $e_\nN$.
Therefore, at time step $n$ of our sampling algorithm, we will have a subset of our tree
\begin{equation}\label{eq:subtree}
    \forestn{n} = [e_{1}, \dots , e_{n-1}]
\end{equation}
Note that $\forestn{1}=\emptyset$ and $\forestn{\nN+1}=\tree$.
We can then express the probability of a tree as
\begin{equation}\label{eq:cond-prob}
    \prob(\tree)=\prod_{n=1}^{\nN}\prob(e_n \mid \tree_{< n})
\end{equation}

We sample the first edge $e_1$ with probability $\prob(e_1)$.
We can find its marginal probability by taking the derivative of the log partition function $\log \Z$ in \cref{thm:mtt} which \citet{koo-et-al-2007} show to be\footnote{\label{footnote:marginal}For spanning trees, the marginal probability can be similarly derived as $\marginal{i}{j}=\wedgeij\left(\cachedelem{jj} - \truth{i\neq\root}\cachedelem{ij}\right)$ where $\cached$ is the transpose of the inverse of the Laplacian matrix in \cref{footnote:lap}.}
\begin{align}\label{eq:marginal}
    &\marginal{i}{j} \\ & = \begin{cases}
    \wedge{\root}{j}\cachedelem{1 j} & \textbf{ if } i=\root \\
    \wedgeij\left(\truth{j\neq 1}\cachedelem{jj} - \truth{i\neq 1}\cachedelem{ij}\right)  & \textbf{ otherwise}
    \end{cases} \nonumber
\end{align}
where $\cached=\lap^{-\top}$ (the transpose of the inverse of $\lap$) and $\truth{x}=1\iff x$ is true, otherwise, $\truth{x}=0$.
Therefore, after computing $\cached$ once, we can compute each $\marginal{i}{j}$ in $\bigo{1}$ time.
Finding $\cached$ requires us to take a matrix inverse, and so runs in $\bigo{\nN^3}$ time.\footnote{This runtime is also true by automatic differentiation \citep{griewank-walther} as finding $\Z$ takes $\bigo{\nN^3}$ time.}

Each subsequent edge that we sample must be conditioned by all previously sampled edges ($\forestn{n}$).
At the $n\th$ step, we have sampled $\forestn{n}$, and so our final sampled tree $\tree$ will be such that $\forestn{n}\subseteq\tree$.
Therefore, sampling $e_n$ from $\graph$ is equivalent to sampling $e_n$ from the subgraph $\treeinclude{\graph}{\forestn{n}}$, which is defined as the largest graph such that $\tree\in\treesg{\treeinclude{\graph}{\forestn{n}}}\implies\forestn{n}\subseteq\tree$.
Consequently, if $\edgeij\in\forestn{n}$, then $\treeinclude{\graph}{\forestn{n}}$ does not contain any other incoming edges to node $j$ other than $\edgeij$.\footnote{When sampling dependency trees, one would think that if $\edge{\root}{j}\in\forestn{n}$, we would need to remove all outgoing edges from the root. However, by the construction of our Laplacian, $\Z$ only accounts for dependency trees and so the marginals already enforce this restriction. Therefore, we only need to remove all other incoming edges to $j$.}

A correct ancestral sampling algorithm will sample an edge $e$ for each non-root node from the graph using \cref{eq:marginal}, and then update graph to be $\treeinclude{\graph}{e}$ and repeat.
This algorithm will have to recompute $\cached$ $\bigo{\nN}$ times and so will have a runtime of $\bigo{\nN^4}$.
We show a graphical example of the algorithm in \cref{fig:example}.
We give pseudocode for this as $\colbournAlg$ in \cref{alg:colbourn}.
The function $\sampleEdge$ samples from the distribution defined in \cref{eq:marginal} and the function $\condition$ updates $\cached$ to contain the transpose of the Laplacian inverse of the conditioned graph.
We describe an efficient procedure for this conditioning step in the following section.

\subsection{Efficiently Computing Marginals}\label{sec:fast-marginals}
\citet{colbourn96} show that we can update the marginals in $\bigo{\nN^2}$ rather than $\bigo{\nN^3}$ by using rank-one updates on $\lap$.
Namely, \citet{colbourn96}'s adds an outer-product $\vu\vv^\top$ to $\lap$ for some column vectors $\vu,\vv\in\real^{\nN}$ for each conditioning operation.
We extend this to the weighted Laplacian for dependency trees.

\input{figures/colbourn}

\begin{lemma}\label{lemma:lape}
For any graph $\graph$ with Laplacian $\lap$ and any edge $e=\edgeij\in\edges$, the $j\th$ column of the Laplacian $\lape$ of $\treeinclude{\graph}{e}$ is given by
\begin{equation}
    \lapeelem{[:,j]} = \begin{cases}
    \wedge{\root}{j}\onehot{1} & \textbf{ \emph{if} } i=\root \\
    \wedge{i}{j}(\truth{j\neq 1}\onehot{j} - \truth{i\neq 1}\onehot{i})
    & \textbf{ \emph{otherwise}}  \\
    \end{cases}
\end{equation}
where
$\onehot{j}$ is the one-hot vector such that the $j\th$ element is $1$.
Furthermore, the $k\th$ column $\lape$, where $k\neq j$, is equivalent to the $k\th$ column of $\lap$.
\end{lemma}
\begin{proof}
Consider the column $\lapeelem{[:,k]}$.

\begin{enumerate}[itemsep=0pt, topsep=0pt, leftmargin=10pt]
    \item[] \case{Case $k=j$} Then the only incoming edge to $j$ in $\treeinclude{\graph}{e}$ is $\edge{i}{j}$.
    \begin{enumerate}[itemsep=0pt, topsep=0pt, leftmargin=10pt]
        \item[] \case{Case $i=\root$} Then element $\lapeelem{1j}=\wedge{\root}{j}$ by \citet{koo-et-al-2007}'s construction.
        As there are no other incoming edges to $kj$, the remainder of the column is filled with zeros ans so $\lapeelem{[:,j]}=\wedge{\root}{j}\onehot{1}$.
        
        \item[] \case{Case $i\neq\root$} Then $\lapeelem{ij}=-\wedgeij$ as long as $i\neq 0$. Since $i\neq\root$, we also have that $\lapeelem{jj}=\wedgeij$ as long as $j\neq 0$. Therefore, we can represent the $j\th$ column by $\wedge{i}{j}(\truth{j\neq 1}\onehot{j} - \truth{i\neq 1}\onehot{i})$.
    \end{enumerate}
    
    \item[] \case{Case $k\neq j$}
    Then all incoming edges to node $k$ are still in $\treeinclude{\graph}{e}$ and so $\lapeelem{[:,k]}=\lapelem{[:,k]}$.
\end{enumerate}

\end{proof}

\cref{lemma:lape} shows that conditioning by an edge $\edgeij$ is equivalent to a column replacement for $\lap$.
A column replacement is equivalent to a rank-one update where $\vv=\onehot{j}$ and $\vu=\wedgeij(\onehot{j} - \onehot{i}) - \lapelem{[:,j]}$ and $\lapelem{[:,j]}$ is the $j\th$ column of $\lap$.
Performing such a rank-one update speeds up the conditioning of $\lap$ from $\bigo{\nN^2}$ to $\bigo{\nN}$.
More importantly, it lets us update $\cached$ in $\bigo{\nN^2}$ using the Sherman--Morrison formula \citep{sherman1950adjustment}, which states that for any matrix $\mat{A}\in\real^{\nN\times\nN}$ and column vectors $\vu,\vv\in\real^{\nN}$\footnote{The Sherman--Morrison formula can be computed in $\bigo{\nN^2}$ due to the associativity of matrix multiplication.}
\begin{equation}
    \inv{(\mat{A}+\vu\vv^\top)} = \inv{\mat{A}} - \mat{T}
\end{equation}
where
\begin{equation}
    \mat{T} = \frac{\inv{\mat{A}}\vu\,\vv^{\top}\inv{\mat{A}}}{1+\vv^{\top}\inv{\mat{A}}\vu}
\end{equation}

Recalling that $\cached$ requires the inverse transpose, and our choice of $\vv=\onehot{j}$, 
we can simplify the expression for $\mat{T}^\top$ to be
\begin{subequations}
\begin{align}
    \mat{T}^\top &= \left(\frac{\inv{\lap}\vu\,\onehot{j}^{\top}\inv{\lap}}{1+\onehot{j}^{\top}\inv{\lap}\vu}\right)^{\top} \\
    &= \left(\frac{\inv{\lap}\vu\,\inv{\lap}_{[j,:]}}{1+\inv{\lap}_{[j,:]}\vu}\right)^{\top} \\
    &= \frac{\lap^{-\top}_{[j,:]}\vu^{\top}\lap^{-\top}}{1+\vu^\top \lap^{-\top}_{[j,:]}} \\
    &= \frac{\cachedelem{[j,:]}\vu^{\top}\cached}{1+\vu^\top\cachedelem{[j,:]}}
\end{align}
\end{subequations}
\noindent Therefore, we can update $\cached$ in $\bigo{\nN^2}$.
We give pseudocode for this efficient update as $\condition$ in \cref{alg:colbourn}.

\begin{thm}\label{thm:colbourn}
For any graph $\graph$, $\algCall{\colbournAlg}{\graph}$ samples a dependency tree with probability
\begin{equation}
    \prob(\tree)=\frac{1}{\Z}\prod_{\edgeij\in\tree}\wedgeij
\end{equation}
Furthermore, $\colbournAlg$ runs in $\bigo{\nN^3}$ time.
\end{thm}
\begin{proof}
The probability of a tree $\prob(\tree)$ can equivalently be written as the product of the conditional edge marginals as in \cref{eq:cond-prob}.
To prove correctness, we prove by induction that for all $n\le\nN$, at the $n\th$ call to $\sampleEdge$, $\cached$ contains the transpose of the Laplacian inverse of $\treeinclude{\graph}{\forestn{n}}$ and an edge $e$ with probability $\prob(e \mid \forestn{n})$.

\begin{enumerate}[itemsep=0pt, topsep=0pt, leftmargin=10pt]
\item[] \case{Base case} Then $n=1$ and $\forestn{1}=\emptyset$. $\cached$ contains the transpose of the Laplacian inverse of $\graph$ as expected and so $\sampleEdge$ will sample an edge $e$ with probability $\prob(e)$ as expected.

\item[] \case{Inductive step} Then $\forestn{n}=[e_1,\dots,e_{n-1}]$.
At the $(n-1)\th$ call to $\sampleEdge$, $\cached$ contains the transpose of the Laplacian inverse of $\treeinclude{\graph}{\forestn{n-1}}$.
We then call $\condition$, which by \cref{lemma:lape} and the Sherman--Morrison formula updates $\cached$ to be the transpose of the Laplacian inverse of $\treeinclude{\graph}{\forestn{n}}$.
Therefore, at the $n\th$ call to $\sampleEdge$, $\cached$ contains the correct values and so $\sampleEdge$ will sample an edge $e$ with probability $\prob(e \mid \forestn{n})$.
\end{enumerate}

Therefore, $\colbournAlg$ samples a dependency tree $\tree$ with the correct probability.
We have $\nN$ iterations of the main loop, each call to $\sampleEdge$ takes $\bigo{\nN}$ times and each call to $\condition$ takes $\bigo{\nN^2}$ time.
These runtimes are easily observed from the pseudocode.
Therefore, $\colbournAlg$ has a runtime of $\bigo{\nN^3}$.
\end{proof}

\input{figures/experiment}

\subsection{Runtime Experiment}\label{sec:experiment}
We conduct a brief runtime experiment for $\colbournAlg$ and $\wilsonRCAlg$ (see \cref{sec:wilson}) whose runtimes are $\bigo{\nN^3}$ and $\bigo{\hittingTime}$ respectively.
We artificially generate random complete graphs of increasing size and measure the average sample time of each algorithm.\footnote{The experiment was conducted using an Intel(R) Core(TM) i7-7500U processor with 16GB RAM.}
The results of the experiment are shown in \cref{fig:experiment}.
We note that despite $\colbournAlg$ being slower, the best-fit line for $\colbournAlg$'s runtime has a slope of $1.42$, suggesting it is much faster in practice than its complexity bound $\bigo{\nN^3}$.\footnote{We would expect the slope to be $\approx 3$ to match the complexity bound.}

\section{Sampling Without Replacement}\label{sec:swor}
In this section, we present a novel extension to $\colbournAlg$ that can sample dependency trees without replacement.
SWOR algorithms are useful when we must sample multiple trees from the same graph.
Specifically, when the distribution of trees over the graph is skewed so that a normal sampling algorithm frequently samples the same trees.
This is often the case when the edge weights have been learned using a neural model \citep{dozat, ma}.
The SWOR algorithm we present follows the scheme of \citet{shi2020}.

In order to use $\colbournAlg$ to sample without replacement, we need an expression of the edge marginals conditioned on the set of previously sampled trees.
If $\treeset$ is the set of previously sampled trees, then we need to compute the following marginal probability efficiently
\begin{equation}
    \prob(\bedgeij \mid \treeset) = \frac{1}{\Z_{\treeset}} \sum_{\tree\in\dtrees_{ij}\setminus\treeset} \weight{\tree}
\end{equation}
where $\dtrees_{ij}$ is the set of all dependency trees that contain edge $\edgeij$ and
\begin{equation}
    \Z_{\treeset} \defeq\  \monstersum{\tree\in\dtrees\setminus\treeset}\ \  \weight{\tree} = \Z - \sum_{\tree\in\treeset} \weight{\tree}
\end{equation}

\begin{lemma}\label{lemma:swor}
For any graph $\graph$, set of trees $\treeset$, and edge $\edgeij\in\edges$,
\begin{align}\label{eq:new-marginal}
\!\!\prob(\bedgeij \mid \treeset) = \frac{1}{\Z_{\treeset}} \! \left[ \Z\!\cdot\! \prob(\bedgeij) \!-\!\! \sum_{\tree \in \treesetij}\! \weight{\tree} \right]\!
\end{align}
where $\treesetij\subseteq\treeset$ is the set of trees in $\treeset$ that contain the edge $\edgeij$.
\end{lemma}
\begin{proof}
\begin{subequations}
\begin{align}
    &\prob(\bedgeij \mid \treeset) \\
    &= \frac{1}{\Z_{\treeset}} \sum_{\tree\in\dtrees_{ij}\setminus\treeset} \weight{\tree} \\
    &= \frac{1}{\Z_{\treeset}}\left[ \sum_{\tree\in\dtrees_{ij}} \weight{\tree} - \sum_{\tree\in\treeset_{ij}} \weight{\tree} \right] \\
    &= \frac{1}{\Z_{\treeset}} \left[ \Z\, \prob(\bedgeij) - \sum_{\tree \in \treeset_{ij}} \weight{\tree} \right]
\end{align}
\end{subequations}
\end{proof}
\noindent \cref{lemma:swor} then gives a new formula to use for sampling edges by re-weighting the probability of an edge.\footnote{Re-weighting has been recently used by \citet{stanojevic-cohen-2021} to speed-up algorithms for single-root tree decoding algorithms.}
We can then compute the marginal distribution for an incoming edges to a node in $\bigo{\nN + \nK}$ time. 
Note that \cref{eq:new-marginal} makes explicit use of $\Z$ which is not needed for the original marginals in \cref{eq:marginal}.
Consequently, as we sample an edge from the tree, we must condition $\Z$ as well as $\cached$.
Fortunately, this can be done in $\bigo{\nN}$ using the matrix determinant Lemma, which states that for any matrix $\mat{A}\in\real^{\nN\times\nN}$ and column vectors $\vu,\vv\in\real^{\nN}$
\begin{equation}
    \abs{\mat{A}+\vu\vv^\top} = \abs{\mat{A}}(1+\vv^{\top}\inv{\mat{A}}\vu)
\end{equation}
Furthermore, at each conditioning step, we must also update $\treeset$ (and $\Z_{\treeset}$), to only include the sampled trees containing the new sampled edge.
These can both be achieved in $\bigo{\nK}$ time where $\nK$ is the number of trees that we sample.
The pseudocode for the sampling and conditioning steps are given as $\sampleEdgeSwor$ and $\conditionSwor$ in \cref{alg:swor}.
The sampling algorithm its self, $\sworAlg$ is similar to $\colbournAlg$ in \cref{alg:colbourn}.
However, it samples $\nK$ dependency trees rather than a single dependency tree and stores additional variables in order to cache frequently used values such as the original $\Z$, $\lap$, and $\cached$ values.

\input{figures/swor}

\begin{thm}\label{thm:swor}
For any graph $\graph$ and $\nK > 0$, $\algCall{\sworAlg}{\graph, \nK}$ samples $\nK$ dependency trees without replacement, where $\tree$ is sampled with probability
\begin{equation}\label{eq:swor}
    \prob(\tree \mid \treeset)=\frac{\truth{\tree\not\in\treeset}}{\Z_{\treeset}}\prod_{\edgeij\in\tree}\wedgeij
\end{equation}
where $\treeset$ is the set of trees sampled prior to $\tree$.
Furthermore, $\sworAlg$ runs in $\bigo{\nK\nN^3 + \nK^2\nN}$ time.
\end{thm}
\begin{proof}
To prove correctness, we prove by induction that the $\nK\th$ sampled tree is sampled with the probability in \cref{eq:swor}.

\begin{enumerate}[itemsep=0pt, topsep=0pt, leftmargin=10pt]
\item[] \case{Base case} Then $\nK=1$ and so $\treeset'=\emptyset$ and $\Z'_{\treeset}=\Z'$.
Therefore, $\sworAlg$ executes $\colbournAlg$ and samples tree $\tree$ with probability $\prob(\tree)$ as expected.

\item[] \case{Inductive step}
Assume that $\treeset'$ contains $\nK-1$ trees, which were each sampled with the correct probability.
Then $\Z'_{\treeset}=\Z'-\sum_{\tree\in\treeset'}\weight{\tree}$ and so by \cref{lemma:swor}, $\sampleEdgeSwor$ will sample the first edge of the new tree with the correct probability.
We can then prove that all edges of the new tree are sampled with the correct probability using an inductive proof analogous to \cref{thm:colbourn}.
\end{enumerate}

We require $\bigo{\nN^3}$ to find $\Z'$ and $\cached'$.
Then for each of the $\nK$ sampled trees, we have $\nN$ iterations of the main loop.
Each call to $\sampleEdgeSwor$ takes $\bigo{\nN+\nK}$ times and each call to $\conditionSwor$ takes $\bigo{\nN^2+\nK}$ time.
These runtimes are easily observed from the pseudocode.
Therefore, $\sworAlg$ has a runtime of $\bigo{\nK\nN^3 + \nK^2\nN}$.
\end{proof}

\section{Conclusion}
We presented two efficient approaches to sample spanning trees subject to a root constraint, which were based on prior algorithms by \citet{wilson96} and \citet{colbourn96}.
While \citet{wilson96}'s $\bigo{\hittingTime}$ algorithm was more rapid, \citet{colbourn96}'s $\bigo{\nN^3}$ algorithm is extendable to a novel sampling without replacement algorithm.
To the best of our knowledge, not much work has been done in graph-based dependency parsing to sample dependency trees, and none has used sampling without replacement.
We hope that this paper serves as a tutorial for how this can be done and encourages the use of sampling in future work.\looseness=-1

\section*{Acknowledgments}
We would like to thank all the reviewers for their invaluable feedback and time spent engaging with our work. The first author is supported by the University of Cambridge School of Technology Vice-Chancellor's Scholarship as well as by the University of Cambridge Department of Computer Science and Technology's EPSRC.

\bibliography{emnlp2021}
\bibliographystyle{acl_natbib}

\end{document}

%% file: macros.tex
\newcommand{\checkNotation}[1]{{#1}}

\newcommand{\defn}[1]{\textbf{#1}}
\renewcommand{\th}[0]{^{\text{th}}}
\renewcommand{\hat}[1]{\widehat{#1}}
\renewcommand{\bar}[1]{\overline{#1}}
\renewcommand{\setminus}[0]{\smallsetminus}
\newcommand{\defeq}[0]{\mathrel{\stackrel{\textnormal{\tiny def}}{=}}}

\newcommand{\bigo}[1]{\mathcal{O}(#1)}
\newcommand{\abs}[1]{\lvert #1 \rvert}
\newcommand{\weight}[1]{\checkNotation{w\!\left(#1\right)\!}}

\let\emptyset\varnothing

\newcommand{\graph}{\checkNotation{G}}
\newcommand{\tree}{\checkNotation{t}}
\newcommand{\trees}{\checkNotation{\mathcal{T}}}
\newcommand{\treesg}[1]{\checkNotation{\trees(#1)}}
\newcommand{\dtrees}{\checkNotation{\mathcal{D}}}
\newcommand{\dtreesg}[1]{\checkNotation{\dtrees(#1)}}
\renewcommand{\root}{\checkNotation{\rho}}

\newcommand{\forestn}[1]{\checkNotation{\tree_{<#1}}}

\newcommand{\treeinclude}[2]{\checkNotation{#1 \succ #2}}

\newcommand{\edges}[0]{\checkNotation{\mathcal{E}}}
\newcommand{\nodes}[0]{\checkNotation{\mathcal{N}}}

\newcommand{\bedge}[2]{\checkNotation{#1\,{\rightarrow}\,#2}}
\newcommand{\bedgeij}[0]{\bedge{i}{j}}

\newcommand{\edge}[2]{(\bedge{#1}{#2})}
\newcommand{\edgeij}[0]{\edge{i}{j}}

\renewcommand{\wedge}[2]{\checkNotation{w_{{\bedge{#1}{#2}}}}}
\newcommand{\wedgeij}{\wedge{i}{j}}

\newcommand{\marginal}[2]{\checkNotation{\prob(\bedge{#1}{#2})}}

\newcommand{\truth}[1]{\checkNotation{\delta_{#1}}}

\newcommand{\mat}[1]{\checkNotation{\mathbf{#1}}}
\newcommand{\matrixelem}[2]{\checkNotation{\mathrm{#1}_{#2}}}

\newcommand{\inv}[1]{#1^{-1}}
\newcommand{\zerovector}{\mat{0}}

\newcommand{\lap}{\mat{L}}
\newcommand{\lapelem}[1]{\matrixelem{L}{#1}}
\newcommand{\lape}{\mat{L^e}}
\newcommand{\lapeelem}[1]{\matrixelem{L^e}{#1}}
\newcommand{\cached}{\checkNotation{\mathbf{B}}}
\newcommand{\cachedelem}[1]{\checkNotation{\mathrm{B}_{#1}}}

\newcommand{\vu}{\mat{u}}
\newcommand{\vv}{\mat{v}}

\newcommand{\rootv}{\checkNotation{\boldsymbol \root}}

\newcommand{\treeset}{\checkNotation{D}}
\newcommand{\treesetij}{\checkNotation{D_{ij}}}
\newcommand{\hittingTime}{\checkNotation{H}}

\newcommand{\Z}{\checkNotation{\mathrm{Z}}}
\newcommand{\prob}{\checkNotation{p}}

\newcommand{\onehot}[1]{\checkNotation{\overrightarrow{\boldsymbol{1}_{#1}}}}

\newcommand{\timesequal}{{\,*\textsf{=}\,}}
\newcommand{\plusequal}{{\,\textsf{+=}\,}}
\newcommand{\minusequal}{{\,\textsf{-=}\,}}

\newcommand{\case}[1]{\checkNotation{\noindent\emph{#1:}}}

\newcommand{\nN}{\checkNotation{N}}
\newcommand{\nK}{\checkNotation{K}}

\newcommand{\real}{\mathbb{R}}

\newcommand{\nonnegreal}{\mathbb{R}_{\ge 0}}

\newcommand{\algFace}[1]{\texttt{#1}}
\newcommand{\algCall}[2]{#1\!\left( #2 \right)}

\newcommand{\colbournAlg}{\algFace{colbourn}}

\newcommand{\sworAlg}{\algFace{swor}}

\newcommand{\wilsonAlg}{\algFace{wilson}}
\newcommand{\wilsonRCAlg}{\algFace{wilson\_rc}}

\newcommand{\sampleEdge}{\algFace{sample\_edge}}
\newcommand{\sampleEdgeSwor}{\algFace{sample\_edge}'}
\newcommand{\sampleEdgeCall}[1]{\algCall{\sampleEdge}{#1}}

\newcommand{\condition}{\algFace{condition}}
\newcommand{\conditionSwor}{\algFace{condition}'}
\newcommand{\conditionCall}[1]{\algCall{\condition}{#1}}

\newcommand{\lapAlg}[0]{\algFace{Laplacian}}
\newcommand{\lapAlgCall}[1]{\algCall{\lapAlg}{#1}}

\newcommand{\monstersum}[1]{{\sum\limits_{\mathclap{\substack{#1}}}}}  
\newcommand{\monsterprod}[1]{{\prod_{\mathclap{\substack{#1}}}}}  

\newcommand{\ars}{arborescences\xspace}

\newcommand*\nodeId[1]{\tikz[baseline]{
            \node[shape=circle,draw,inner sep=1pt, text centered, text depth=0.2mm] () at (0,0.07) {\tiny $#1$};}}

\definecolor{antiquefuchsia}{rgb}{0.57, 0.36, 0.51}
\definecolor{newline}{rgb}{0.53, 0.15, 0.34}
\definecolor{darkslategray}{rgb}{0.18, 0.31, 0.31}
\definecolor{darkspringgreen}{rgb}{0.09, 0.45, 0.27}
\definecolor{britisharmygreen}{rgb}{0.0, 0.26, 0.15}
\newcommand{\globalvariable}[1]{{\color{darkspringgreen} #1}}

%% file: figures/walk.tex
\tikzset{cross/.style={cross out, draw=black, minimum size=2*(#1-\pgflinewidth), inner sep=0pt, outer sep=0pt}, cross/.default={2pt}}

\tagmcbegin{tag=Figure,alttext={The figure contains two random walks. The first random in subfigure a contains the walk of the following nodes: 1, 2, 3, 4, 2, 5, and then root. The cycle from node 3 to 4 to 2 is circled in red. Subfigure b contains the same walk but without the cycle. That is, we have the tree path: 1, 2, 5, and then root.}}
\begin{figure*}[t!]
    \centering
    \begin{subfigure}[b]{.9\linewidth}
\centering
\begin{tikzpicture}
\begin{scope}[every node/.style={circle,thick,draw, inner sep=3pt}]
    \node (1) at (0, 0) {\Large $1$};
    \node (2) at (2, 0) {\Large $2$};
    \node (3) at (4, 0) {\Large $3$};
    \node (4) at (6, 0) {\Large $4$};
    \node (5) at (8, 0) {\Large $2$};
    \node (6) at (10, 0) {\Large $5$};
    \node (7) at (12, 0) {\Large $\root$};
\end{scope}
\begin{scope}[>=latex,
              every node/.style={fill=white, inner sep=0.5pt},
              every edge/.style={draw, thick}]
    \path [->] (1) edge[] (2);
    \path [->] (2) edge[] (3);
    \path [->] (3) edge[] (4);
    \path [->] (4) edge[] (5);
    \path [->] (5) edge[] (6);
    \path [->] (6) edge[] (7);
\end{scope}
\fill[fill=red, opacity=0.1] (6,0) ellipse (3cm and 1cm);
\end{tikzpicture}
\caption{Random walk path from node $\nodeId{1}$ to $\nodeId{\root}$.}
\label{walk:a}
\end{subfigure}

\hspace{30pt}

\begin{subfigure}[b]{.9\linewidth}
\centering
\begin{tikzpicture}
\begin{scope}[every node/.style={circle,thick,draw, inner sep=3pt}]
    \node(1) at (0, 0) {\Large $1$};
    \node (2) at (2, 0) {\Large $2$};
    \node (6) at (4, 0) {\Large $5$};
    \node (7) at (6, 0) {\Large $\root$};
\end{scope}
\begin{scope}[>=latex,
              every node/.style={fill=white, inner sep=0.5pt},
              every edge/.style={draw, thick}]
    \path [->] (1) edge[] (2);
    \path [->] (2) edge[] (6);
    \path [->] (6) edge[] (7);
\end{scope}
\end{tikzpicture}
\caption{Tree path from node $\nodeId{1}$ to $\nodeId{\root}$.}
\label{walk:b}
\end{subfigure}

\hspace{20pt}

    \caption{
    Cycle erasure in a random walk of a graph.
    The associated graph of the above walk has six nodes (including the root node $\nodeId{\root}$) and we start the walk at $\nodeId{1}$ with the tree only containing the root node.
    The random walk includes a cycle with nodes $\nodeId{2}$, $\nodeId{3}$, and $\nodeId{4}$, this cycle is erased when we create the path from $\nodeId{1}$ back up to $\nodeId{\root}$.
    Note that the arrows here mark the path upwards rather than the edges in the tree, the tree edges are reversed (e.g, $\nodeId{\root}\rightarrow\nodeId{5}$).
    }
    \label{fig:walk}
\end{figure*}
\tagmcend

%% file: figures/wilson.tex
\begin{figure}[t!]
    \centering
    \begin{algorithmic}[1]
    \Func{$\algCall{\wilsonAlg}{\graph}$}
    \LinesComment{Sample a spanning tree from a graph $\graph$; requires $\bigo{\hittingTime}$ time, $\bigo{\nN^2}$ space.}
        \State $\tree \gets \zerovector$
        \State $\mathrm{visited}\gets \{ \root \}$
        \LinesComment{The following for and while loops take $\bigo{\hittingTime}$ to execute. More specifically, \cref{line:sample} is called $\bigo{\hittingTime}$ times.}
        \For{$i\in\nodes\setminus\{\root\}$}
            \State $u\gets i$
            \While{$u\not\in\mathrm{visited}$}
                \State\label{line:sample} Sample $v\!\in\!\nodes$ with weight $w'_{\bedge{v}{u}}$ 
                \State $\tree_{u}\gets v$
                \State $u \gets v$
            \EndWhile
            \State $u\gets i$
            \While{$u\not\in\mathrm{visited}$}
                \State $\mathrm{visited}.\mathrm{add}(u)$
                \State $u \gets v$ such that $\tree_{u}=v$ 
            \EndWhile
        \EndFor
        \State \Return $\tree$
    \EndFunc
    \end{algorithmic}
    \caption{\citet{wilson96}'s algorithm to sample spanning trees from a graph.}
    \label{alg:wilson}
\end{figure}

%% file: figures/wilson_rc.tex
\begin{figure}[t!]
    \centering
    \begin{algorithmic}[1]
    \Func{$\algCall{\wilsonRCAlg}{\graph}$}
    \LinesComment{Sample a dependency tree from a graph $\graph$; requires $\bigo{\hittingTime}$ time, $\bigo{\nN^2}$ space.}
        \State {\color{newline} Sample $j$ with weight $\wedge{\root}{j}$ \label{line:new}}
        \State $\tree \gets \zerovector$
        \State {\color{newline} $\tree_j \gets \root$}
        \State $\mathrm{visited}\gets\{ j \}$ \label{line:start}
        \LinesComment{The following for and while loops take $\bigo{\hittingTime}$ to execute. More specifically, \cref{line:sample2} is called $\bigo{\hittingTime}$ times.}
        \For{$i\in\nodes\setminus\{\root\}$}
            \State $u\gets i$
            \While{$u\not\in\mathrm{visited}$}
                \State\label{line:sample2} {\color{newline}Sample $v\!\in\!\nodes\!\setminus\!\{\root\}$ with weight $w'_{\bedge{v}{u}}$}
                \State $t_{u}\gets v$ 
                \State $u \gets v$
            \EndWhile
            \State $u\gets i$
            \While{$u\not\in\mathrm{visited}$}
                \State $\mathrm{visited}.\mathrm{add}(u)$
                \State $u \gets v$ such that $t_{u}=v$
            \EndWhile
        \EndFor
        \State \Return $\tree$ \label{line:end}
    \EndFunc
    \end{algorithmic}
    \caption{Modification of \citet{wilson96}'s algorithm to sample dependency trees from a biased distribution. The lines that differ to $\wilsonAlg$ are {\color{newline}highlighted}.}
    \label{alg:wilson-rc}
\end{figure}

%% file: figures/example.tex
\tikzset{cross/.style={cross out, draw=black, minimum size=2*(#1-\pgflinewidth), inner sep=0pt, outer sep=0pt}, cross/.default={2pt}}

\tagmcbegin{tag=Figure,alttext={The image contains five directed graphs. All five graphs have five nodes, labeled 1, 2, 3, 4, and root. The first graph is a fully connected graph where all incoming edges to node 1 are highlighted and the edge from node 4 to 1 is marked as edge e1. The second graph only has e1 directed to node 1 and all edges to node 2 are highlighted with the edge from the root to node 2 marked as e2. The third graph only has e1 directed to node 1 and e2 directed to node 2. All edges to node 3 are highlighted with the edge from node 2 to 3 marked as e3. The fourth graph only has e1, e2, and e3 directed to nodes 1 to 3 respectively, and all edges to node 4 are highlighted, with the edge from node 2 to 4 marked as e4. The final graph only contains the edges e1, e2, e3, and e4.}}
\begin{figure*}[t!]
    \centering
    \begin{subfigure}[b]{.19\linewidth}
\centering
\begin{tikzpicture}
\begin{scope}[every node/.style={circle,thick,draw, inner sep=2pt}]
    \node(r) at (0, 0) {$\root$};
    \node (1) at (-1, 1) {$1$};
    \node (2) at (1, 1) {$2$};
    \node (3) at (1, -1) {$3$};
    \node (4) at (-1, -1) {$4$};
\end{scope}
\begin{scope}[>=latex,
              every node/.style={fill=white, inner sep=0.5pt},
              every edge/.style={draw, thick}]
    \path [->] (r) edge[dashed] (1);
    \path [->] (r) edge[] (2);
    \path [->] (r) edge[] (3);
    \path [->] (r) edge[] (4);
    
    \path [->] (1) edge[] (2);
    \path [->] (1) edge[bend left=30] (3);
    \path [->] (1) edge[bend left=15] (4); 
    
    \path [->] (2) edge[dashed, bend left=15] (1);
    \path [->] (2) edge[] (3);
    \path [->] (2) edge[bend left=30] (4);
    
    \path [->] (3) edge[dashed, bend left=30] (1);
    \path [->] (3) edge[bend left=15] (2);
    \path [->] (3) edge[] (4); 
    
    \path [->] (4) edge[dashed] node {\small $e_1$} (1);
    \path [->] (4) edge[bend left=30] (2);
    \path [->] (4) edge[bend left=15] (3); 
\end{scope}
\end{tikzpicture}
\caption{$\graph$}
\label{subfig:a}
\end{subfigure}
\begin{subfigure}[b]{.19\linewidth}
\centering
\begin{tikzpicture}
\begin{scope}[every node/.style={circle,thick,draw, inner sep=2pt}]
    \node(r) at (0, 0) {$\root$};
    \node (1) at (-1, 1) {$1$};
    \node (2) at (1, 1) {$2$};
    \node (3) at (1, -1) {$3$};
    \node (4) at (-1, -1) {$4$};
\end{scope}
\begin{scope}[>=latex,
              every node/.style={fill=white, inner sep=0.5pt},
              every edge/.style={draw, thick}]
    \path [->] (r) edge[dashed] node {\small $e_2$} (2);
    \path [->] (r) edge[] (3);
    \path [->] (r) edge[] (4);
    
    \path [->] (1) edge[dashed] (2);
    \path [->] (1) edge[bend left=30] (3);
    \path [->] (1) edge[bend left=15] (4); 
    
    \path [->] (2) edge[] (3);
    \path [->] (2) edge[bend left=30] (4);
    
    \path [->] (3) edge[dashed, bend left=15] (2);
    \path [->] (3) edge[] (4); 
    
    \path [->] (4) edge[] node {\small $e_1$} (1);
    \path [->] (4) edge[dashed, bend left=30] (2);
    \path [->] (4) edge[bend left=15] (3); 
\end{scope}
\end{tikzpicture}
\caption{$\treeinclude{\graph}{\forestn{2}}$}
\label{subfig:b}
\end{subfigure}
\begin{subfigure}[b]{.19\linewidth}
\centering
\begin{tikzpicture}
\begin{scope}[every node/.style={circle,thick,draw, inner sep=2pt}]
    \node(r) at (0, 0) {$\root$};
    \node (1) at (-1, 1) {$1$};
    \node (2) at (1, 1) {$2$};
    \node (3) at (1, -1) {$3$};
    \node (4) at (-1, -1) {$4$};
\end{scope}
\begin{scope}[>=latex,
              every node/.style={fill=white, inner sep=0.5pt},
              every edge/.style={draw, thick}]
    \path [->] (r) edge[] node {\small $e_2$} (2);
    \path [->] (r) edge[dashed] (3);
    \path [->] (r) edge[] (4);
    
    \path [->] (1) edge[dashed, bend left=30] (3);
    \path [->] (1) edge[bend left=15] (4); 
    
    \path [->] (2) edge[dashed] node {\small $e_3$} (3);
    \path [->] (2) edge[bend left=30] (4);
    
    \path [->] (3) edge[] (4); 
    
    \path [->] (4) edge[] node {\small $e_1$} (1);
    \path [->] (4) edge[dashed, bend left=15] (3); 
\end{scope}
\end{tikzpicture}
\caption{$\treeinclude{\graph}{\forestn{3}}$}
\label{subfig:c}
\end{subfigure}
\begin{subfigure}[b]{.19\linewidth}
\centering
\begin{tikzpicture}
\begin{scope}[every node/.style={circle,thick,draw, inner sep=2pt}]
    \node(r) at (0, 0) {$\root$};
    \node (1) at (-1, 1) {$1$};
    \node (2) at (1, 1) {$2$};
    \node (3) at (1, -1) {$3$};
    \node (4) at (-1, -1) {$4$};
\end{scope}
\begin{scope}[>=latex,
              every node/.style={fill=white, inner sep=0.5pt},
              every edge/.style={draw, thick}]
    \path [->] (r) edge[] node {\small $e_2$} (2);
    \path [->] (r) edge[dashed] (4);
    
    \path [->] (1) edge[dashed, bend left=15] (4); 
    
    \path [->] (2) edge[] node {\small $e_3$} (3);
    \path [->] (2) edge[dashed, bend left=30] node {\small $e_4$} (4);
    
    \path [->] (3) edge[dashed] (4); 
    
    \path [->] (4) edge[] node {\small $e_1$} (1);
\end{scope}
\end{tikzpicture}
\caption{$\treeinclude{\graph}{\forestn{4}}$}
\label{subfig:d}
\end{subfigure}
\begin{subfigure}[b]{.19\linewidth}
\centering
\begin{tikzpicture}
\begin{scope}[every node/.style={circle,thick,draw, inner sep=2pt}]
    \node(r) at (0, 0) {$\root$};
    \node (1) at (-1, 1) {$1$};
    \node (2) at (1, 1) {$2$};
    \node (3) at (1, -1) {$3$};
    \node (4) at (-1, -1) {$4$};
\end{scope}
\begin{scope}[>=latex,
              every node/.style={fill=white, inner sep=0.5pt},
              every edge/.style={draw, thick}]
    \path [->] (r) edge[] node {\small $e_2$} (2);
    
    
    \path [->] (2) edge[] node {\small $e_3$} (3);
    \path [->] (2) edge[bend left=30] node {\small $e_4$} (4);
    
    
    \path [->] (4) edge[] node {\small $e_1$} (1);
\end{scope}
\end{tikzpicture}
\caption{$\treeinclude{\graph}{\tree}$}
\label{subfig:e}
\end{subfigure}
    \caption{
    Consider sampling a tree from the fully connected graph $\graph$ given in \emph{(a)}.
    We do this by sampling an incoming edge to each non-root node.
    We first sample an incoming edge to $\nodeId{1}$, the possible edges are dashed in \emph{(a)}.
    Suppose we sample $e_1$ with probability $\prob(e_1)$, then we have $\forestn{2}=\{e_1\}$.
    If we include $e_1$ in our graph as in \emph{(b)}, and repeat the process, we will sample edge $e_2$ with probability $\prob(e_2 \mid \forestn{2})$.
    We now have $\forestn{3} = \{e_1, e_2\}$, and we can sample an incoming edge $e_3$ to $\nodeId{3}$ with probability $\prob(e_3 \mid \forestn{3})$ as in \emph{(c)}.
    We can similarly find $\forestn{4}$ in \emph{d}.
    Finally, in \emph{(e)}, we have $\forestn{5}=\tree=\{e_1, e_2, e_3, e_4\}$, which is a tree in $\treesg{\graph}$.
    Note $\treesg{\treeinclude{\graph}{\tree}}=\{\tree\}$.
    }
    \label{fig:example}
\end{figure*}
\tagmcend

%% file: figures/colbourn.tex
\begin{figure}
    \centering

\begin{flushleft}
\hspace{0.5em}
\textbf{Global Variables}: \globalvariable{$\lap$}, \globalvariable{$\cached$}
\end{flushleft}
\begin{algorithmic}[1]
    \Func{$\algCall{\colbournAlg}{\graph}$}
    \LinesComment{Sample a dependency tree from a graph $\graph$; requires $\bigo{\nN^3}$ time, $\bigo{\nN^2}$ space.}
    \State $\globalvariable{\lap} \gets \lapAlgCall{\graph}$ 
    \State $\globalvariable{\cached} \gets \globalvariable{\lap}^{-\top}$ \Comment{$\bigo{\nN^3}$}
    \State $\tree \gets [\, ]$
    \For{ $j\in\nodes\setminus\{\root\}$} 
        \State $e \gets \sampleEdgeCall{j}$
        \State $\tree.\mathrm{append}(e)$ 
        \State $\conditionCall{e}$ 
    \EndFor
    \State \Return $\tree$
    \EndFunc
    
    \Func{$\lapAlgCall{G}$}
        \LinesComment{Construct the Laplacian of \citet{koo-et-al-2007} for dependency trees as in \cref{eq:lap}; requires $\bigo{\nN^2}$ time, $\bigo{\nN^2}$ space.}
        \State $\lap \gets \zerovector$
        \For{$j\in\nodes\setminus\{\root\}$}
            \For{$i\in\nodes\setminus\{\root, j\}$}
                \State $\lapelem{ij} \gets - \wedgeij$
                \State $\lapelem{jj} \plusequal \wedgeij$
            \EndFor
            \State $\lapelem{1j} \gets \wedge{\root}{j}$ \label{line:lap}
            \LinesComment{For spanning trees, we can construct the Laplacian in \cref{footnote:lap} by replacing \cref{line:lap} with $\lapelem{jj}\plusequal\wedge{\root}{j}$.}
        \EndFor
        \State \Return $\lap$
    \EndFunc
    
    \Func{$\sampleEdgeCall{j}$}
    \LinesComment{Sample an incoming edge to $j$ using global variable $\cached$ as in \cref{eq:marginal}; requires $\bigo{\nN}$ time, $\bigo{\nN}$ space.}
    \State $\mat{m} \gets \zerovector$
    \State $\matrixelem{m}{\root} \plusequal \wedge{\root}{j}\globalvariable{\cachedelem{1j}}$ \label{line:root}

    \For{$i\in\nodes\setminus\{\root, j\}$} 
    \State $\matrixelem{m}{i} \plusequal \wedge{i}{j}\left(\delta_{j\neq 1}\globalvariable{\cachedelem{jj}} - \delta_{i\neq 1}\globalvariable{\cachedelem{ij}}\right)$ \label{line:marginal}
    \EndFor
    \LinesComment{For spanning trees, we can construct the marginals in \cref{footnote:marginal} by replacing \cref{line:root} with $\matrixelem{m}{\root}\plusequal\wedge{\root}{j}\cachedelem{jj}$ and \cref{line:marginal} with $\matrixelem{m}{i}\plusequal\wedge{i}{j}(\cachedelem{jj} - \cachedelem{ij})$.}
    \State \Return sample from $\mat{m}$ 
    \EndFunc
    
    \Func{$\conditionCall{e}$}
    \LinesComment{Condition the Laplacian and the transpose of its inverse to always include $e$ in any tree; requires $\bigo{\nN^2}$ time, $\bigo{\nN^2}$ space.}
    \State Let $e=\edgeij$
    \If{$i=\root$}
    \State $\vu \gets \wedge{\root}{j}\onehot{1}  - \globalvariable{\lapelem{[:, j]}}$
    \Else:
    \State $\vu \gets \wedgeij (\truth{j\neq 1}\onehot{j} - \truth{i\neq 1}\onehot{i}) - \globalvariable{\lapelem{[:, j]}}$ 
    \EndIf
    \State $\globalvariable{\lapelem{[:, j]}} \plusequal \vu $
    \State $\globalvariable{\cached} \minusequal {(\globalvariable{\cachedelem{[j, :]}}\vu^\top\globalvariable{\cached)}}\ /\ {(1 + \vu^\top\globalvariable{\cachedelem{[j, :]})}}$ 
    \EndFunc
\end{algorithmic}

    \caption{Algorithm for sampling dependency trees using the method of \citet{colbourn96}.
    We describe the changes required to sample spanning trees in the comments.
    }
    \label{alg:colbourn}
\end{figure}

%% file: figures/experiment.tex
\tagmcbegin{tag=Figure,alttext={The figure contains a plot of the average sample time measured in seconds on a log scale, against the log length of the sentence. The plot contains two lines, one corresponding to Wilson's algorithm and one corresponding to Colbourn's algorithm. Wilson's algorithm has a shallower gradient and a lower intercept and lies beneath Colbourn's line for every measurement.}}
\begin{figure}[t]
    \centering
    \includegraphics[width=0.47\textwidth]{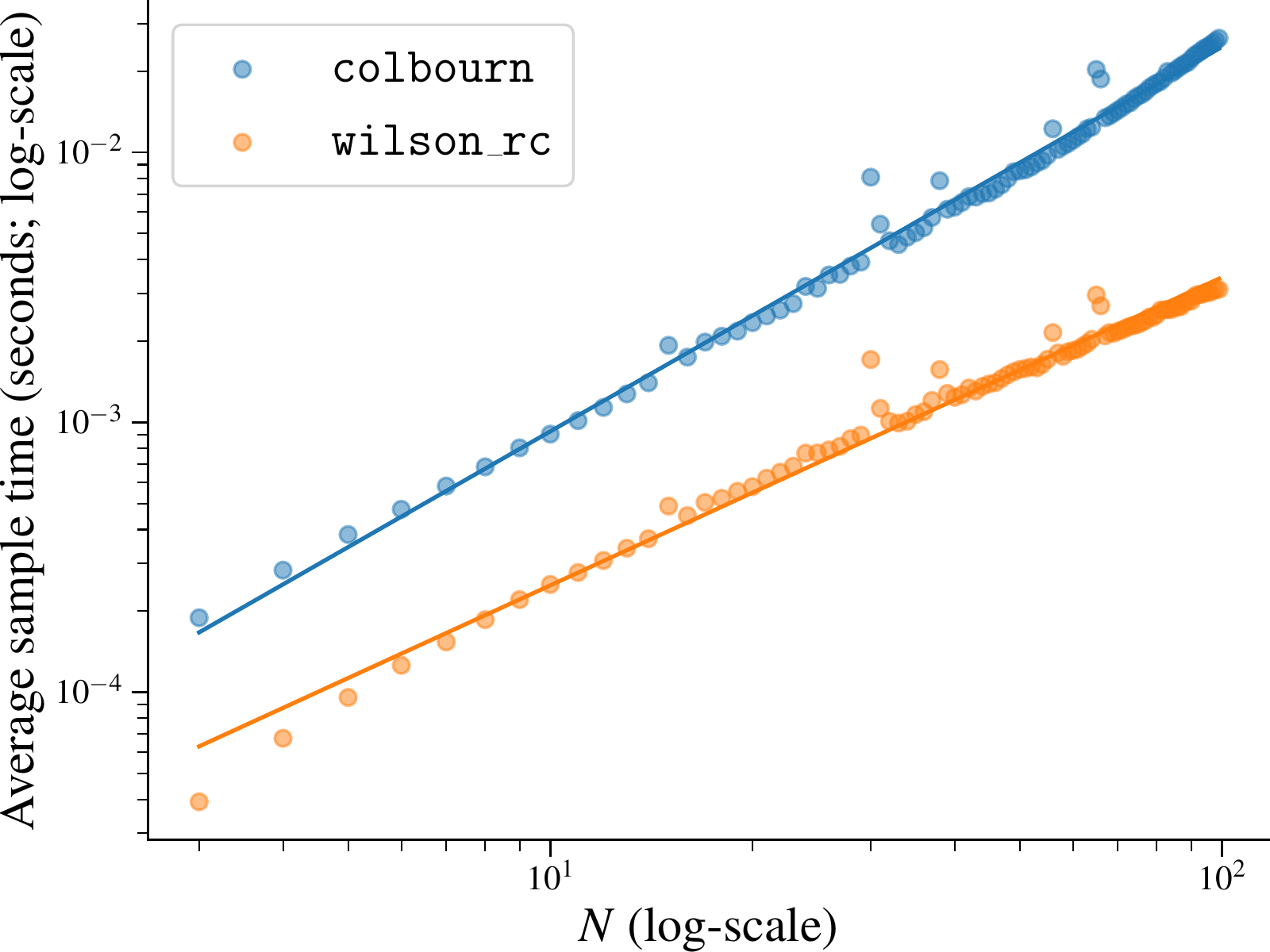}
    \caption{Runtime experiment for sampling using $\wilsonRCAlg$ and $\colbournAlg$. For each graph size, we randomly generated $100$ graphs and took $20$ samples from each graph.
    The best fit lines for $\colbournAlg$ and $\wilsonRCAlg$ have slopes of $1.42$ and $1.14$ respectively.
    }
    \label{fig:experiment}
\end{figure}
\tagmcend

%% file: figures/swor.tex
\begin{figure}[t!]
    \centering

\begin{flushleft}
\hspace{0.5em}
\textbf{Global Variables}: \globalvariable{$\lap$}, \globalvariable{$\cached$}, \globalvariable{$\treeset$}, \globalvariable{$\Z$}, and \globalvariable{$\Z_{\treeset}$}
\end{flushleft}
\begin{algorithmic}[1]
    \Func{$\algCall{\sworAlg}{\graph, \nK}$}
    \LinesComment{Sample $\nK$ dependency trees without replacement from a graph $\graph$; requires $\bigo{\nK\nN^3 + \nK^2\nN}$ time, $\bigo{\nN^2 + \nK\nN}$ space.}
    \State $\lap' \gets \lapAlgCall{\graph}$ 
    \State $\Z' \gets \abs{\lap'};\, \cached' \gets \lap'^{-\top}$ \Comment{$\bigo{\nN^3}$}
    \State $\treeset' \gets [\, ];\, \Z'_{\treeset} \gets \Z'$
    \For{$k\in\{1,\dots,\nK\}$}
    \State $\globalvariable{\treeset} \gets \treeset';\, \globalvariable{\Z} \gets \Z'$
    \State $\globalvariable{\Z_{\treeset}} \gets \Z'_{\treeset}$ 
    \State $\globalvariable{\lap} \gets \lap';\, \globalvariable{\cached} \gets \cached';\, \tree \gets [\, ]$
    \For{ $j\in\nodes\setminus\{\root\}$} 
        \State $e \gets \algCall{\sampleEdgeSwor}{j}$ 
        \State $\tree.\mathrm{append}(e)$ 
        \State $\algCall{\conditionSwor}{e}$ 
    \EndFor
    \State $\treeset'.\mathrm{append}(\tree);\, \Z'_{\treeset} \minusequal \weight{\tree}$ 
    \EndFor
    \State \Return $\treeset'$
    \EndFunc
    
    \Func{$\algCall{\sampleEdgeSwor}{j}$}
    \LinesComment{Sample an incoming edge to $j$ using global variables $\cached$, $\Z$, and $\Z_{\treeset}$ as in \cref{eq:new-marginal}; requires $\bigo{\nN}$ time, $\bigo{\nN}$ space.}
    \State $\mat{m} \gets \zerovector$
    \State $\matrixelem{m}{\root} \plusequal \globalvariable{\Z}\, \wedge{\root}{j}\globalvariable{\cachedelem{1j}} - \sum_{\tree\in T_{\root j}} \weight{\tree}$ \label{line:sub1}
    \For{$i\in\nodes\setminus\{\root, j\}$} \Comment{$\bigo{\nN}$}
    \State $\matrixelem{m}{i} \plusequal \Z\, \wedge{i}{j}\left(\delta_{j\neq 1}\globalvariable{\cachedelem{jj}} - \delta_{i\neq 1}\globalvariable{\cachedelem{ij}}\right)$
    \State $\matrixelem{m}{i} \minusequal \sum_{\tree\in\treesetij} \weight{\tree}$ \label{line:sub2}
    \EndFor
    \State \Return sample from $\frac{1}{\globalvariable{\Z_{\treeset}}} \mat{m}$ 
    \EndFunc

    \Func{$\algCall{\conditionSwor}{e}$}
    \LinesComment{Condition the Laplacian, the transpose of the Laplacian inverse, the partition partition function, and the set of previously sampled trees to always include $e$ in any tree; requires $\bigo{\nN^2+\nK}$ time, $\bigo{\nN^2}$ space.}
    \State Let $e=\edgeij$
    \If{$i=\root$}
    \State $\vu \gets \wedge{\root}{j}\onehot{1} - \globalvariable{\lapelem{[:, j]}}$
    \Else:
    \State $\vu \gets \wedgeij (\truth{j\neq 1}\onehot{j} - \truth{i\neq 1}\onehot{i}) - \globalvariable{\lapelem{[:, j]}}$ %
    \EndIf
    \State $\globalvariable{\lapelem{[:, j]}} \plusequal \vu $ 
    \State $\globalvariable{\Z} \timesequal 1 + \vu^\top\cachedelem{[j, :]}$ 
     \State $\cached \minusequal {(\globalvariable{\cachedelem{[j, :]}}\vu^\top\globalvariable{\cached})}\ /\ {(1 + \vu^\top\globalvariable{\cachedelem{[j, :]})}}$ 
    \State $\globalvariable{\treeset} \gets \treesetij$ 
    \State $\globalvariable{\Z_{\treeset}} \gets \globalvariable{\Z} - \sum_{\tree\in\globalvariable{\treeset}}\weight{\tree}$ 
    \EndFunc
\end{algorithmic}

    \caption{Algorithm for sampling dependency trees without replacement.
    }
    \label{alg:swor}
\end{figure}